\documentclass[oneside,english]{amsart}
\usepackage[T1]{fontenc}
\usepackage[latin9]{inputenc}
\usepackage{float}
\usepackage{wrapfig}
\usepackage{amsthm}
\usepackage{amssymb}
\usepackage{graphicx}

\makeatletter

\floatstyle{ruled}
\newfloat{algorithm}{tbp}{loa}
\providecommand{\algorithmname}{Algorithm}
\floatname{algorithm}{\protect\algorithmname}

\numberwithin{equation}{section}
\numberwithin{figure}{section}
\theoremstyle{plain}
\newtheorem{thm}{\protect\theoremname}
  \theoremstyle{definition}
  \newtheorem{problem}[thm]{\protect\problemname}
  \theoremstyle{definition}
  \newtheorem{example}[thm]{\protect\examplename}
  \theoremstyle{plain}
  \newtheorem{prop}[thm]{\protect\propositionname}
  \theoremstyle{remark}
  \newtheorem{rem}[thm]{\protect\remarkname}
  \theoremstyle{plain}
  \newtheorem{cor}[thm]{\protect\corollaryname}
  \theoremstyle{definition}
  \newtheorem{defn}[thm]{\protect\definitionname}

\@ifundefined{showcaptionsetup}{}{%
 \PassOptionsToPackage{caption=false}{subfig}}
\usepackage{subfig}
\makeatother

\usepackage{babel}
  \providecommand{\corollaryname}{Corollary}
  \providecommand{\definitionname}{Definition}
  \providecommand{\examplename}{Example}
  \providecommand{\problemname}{Problem}
  \providecommand{\propositionname}{Proposition}
  \providecommand{\remarkname}{Remark}
\providecommand{\theoremname}{Theorem}

\begin{document}

\title{How to sample if you must: On Optimal Functional Sampling}

\maketitle
$\;$

\author{Assaf Hallak}

\author{Shie Mannor}
\begin{abstract}
We examine a fundamental problem that models various active sampling
setups, such as network tomography. We analyze sampling of a multivariate
normal distribution with an unknown expectation that needs to be estimated:
in our setup it is possible to sample the distribution from a given
set of linear functionals, and the difficulty addressed is how to
optimally select the combinations to achieve low estimation error.
Although this problem is in the heart of the field of optimal design,
no efficient solutions for the case with many functionals exist. We
present some bounds and an efficient sub-optimal solution for this
problem for more structured sets such as binary functionals that are
induced by graph walks. 
\end{abstract}


\section{Introduction}

Consider a network in which each link has a delay characterized with
some parametric distribution. The network can be probed in order to
find an estimator for these parameters, yet the only measurement obtained
for each probe is the sum of delays along the path. As each probe
costs time, efficiently sampling the network is crucial for estimating
the delays accurately. This example is one of many that can be modeled
by the generative model studied in this paper:
\begin{problem}
Define the following system: $N$ is the number of variables, $x_{t}\in\mathbb{R}^{N}$
is the sample at stage $t$, $y_{t}$ is the measurement produced
in the following manner:
\[
y_{t}=w_{t}^{\top}x_{t},\quad w_{t}\sim N\left(\mu,diag\left\{ \sigma_{i}^{2}\right\} _{i=1}^{N}\right),
\]

where $\sigma_{i}^{2}$ are known. At each stage one may choose $x_{t}$
from a certain subset $X\subseteq\mathbb{R}^{N}$, observe $y_{t}$
and then find an estimator $\hat{\mu}$ using the history of the samples.
The problem is how to choose $x_{t}$ such that the estimator will
have as low error as possible.
\end{problem}
To better understand the problem, we revisit the network tomography
problem (\cite{vardi1996network,coates2002internet}):
\begin{example}
Observe the following network:
\end{example}
Assume that moving through each link in the network results in a random
delay $w\left(e_{i}\right)\sim N\left(\mu_{i},1\right)$. The possible
traces one can probe must start and end in a computer, so only the
traces $C_{1}\rightarrow H\rightarrow C_{2}$,$C_{1}\rightarrow H\rightarrow C_{3}$,
$C_{2}\rightarrow H\rightarrow C_{3}$ and the reverse traces are
available. These can yield the following samples $w\left(e_{1}\right)+w\left(e_{2}\right),w\left(e_{1}\right)+w\left(e_{3}\right),w\left(e_{2}\right)+w\left(e_{3}\right)$.
After sampling once from each trace, it is possible to estimate $\mu_{1},\mu_{2}$
and $\mu_{3}$ with finite expected error. Drawing more samples will
yield an estimator with lower error, but how should one draw them?
Assume for instance $C_{1}\rightarrow H\rightarrow C_{2}$ is sampled
more frequently than the other traces. This may result in a lousy
estimator for $\mu_{3}$ since this probe does not include $e_{3}$.
If the error in each estimator is equivalently important to us, the
optimal policy in this case is not surprisingly probing the network
uniformly over the available traces. However, generally uniform sampling
can generate terrible results.

The paper consists of the following parts: in the next section we
survey previous related works in several fields such as experiment
design, learning theory and network tomography. In Section $3$ we
formulate the problem and discuss its known solution and mathematical
properties. In the succeeding sections, special structured functionals
sets will be considered: initially the general binary case, and afterward
sets generated by graph walks. In graph walks we discuss two setups:
in the first, the random variables are associated with the nodes,
and in the second setup, they are associated with the edges in the
graph (like in Example $2$). In Section $7$ we point\textbf{ }out
the relation to recent works on a specific bandit setup.\textbf{ }The
final chapter will present conclusions, as well as suggestions, for
the ongoing research.

\section{Previous work}

Similar generative models as posed in Problem $1$ have been widely
studied in the field of optimal design (see Pukelsheim \cite{pukelsheim2006optimal}
for an overview of the field). However, as the size of the finite
set $X$ grows, common solutions such as SDP solvers and gradient
techniques are insufficient as their complexity depends on the set
size. This difficulty is recognized in network tomography where each
functional is identified with a trajectory on the graph so that the
size of $X$ can be exponential in the number of variables. Our work
suggests an efficient solution for this particular case.

In machine learning the field of active learning is concerned with
similar problems (for a survey see \cite{settles2010active}). Problem
$1$ highly resembles the multiple linear regression model \cite{montgomery2007introduction},
however unlike regression our work is not focused on estimating the
parameters but rather on choosing samples that will result in a better
estimator. For example, Cohn et al. \cite{cohn1996active} have studied
optimal active learning in various models, including the kernelized
weighted least squares setup. Despite the similarities between the
problems, several key differences in the setup had led to entirely
different mathematical formulations. Our focus is on using the structure
of the set $X$ for obtaining an efficient sampling strategy that
minimizes the estimation error.

We observe an interesting connection to exploration in bandit problems
through the work of Dani et al. \cite{dani2008price} and its follow-up
by Cesa-Bianchi and Lugosi \cite{bianchi2009combinatorial}. They
have come across a key problem similar to ours while proving bounds
on the exploration component of the adversarial online bandit problem
with a restricted linear sampling set. Although these works took great
interest in assessing a similar value function to the one we later
present, their work did not address the optimal sampling issue addressed
by us, nor the computational effort in finding it. As an application
of our work, we solve a specific example mentioned in \cite{bianchi2009combinatorial}.

The main application presented here concerns ``Network Tomography''
(coined by Vardi \cite{vardi1996network}), which deals with inference
on the parameters or topology of a network through probing (for an
overview see Coates et al. \cite{coates2002internet}). In this field
there are many interesting setups, for example finding a network's
structure or some of its unique properties \cite{coates2002maximum,rabbat2004multiple,lawrence2007statistical}.
For instance, in a recent work, Thouin et al. suggested using active
learning in order to infer the bandwidth of a network (\cite{thouin2011large}).
Another related work on parameter estimation was done by Tsang \cite{tsang2005optimal}
who addressed the same problem with different parameters and stresses.
However, most of the works in the field have dealt with more complicated
distributions in different schemes and most of the efforts were put
into finding efficient computation of fine estimators \cite{lawrence2003maximum,coates2001network,shih2001unicast}
rather than on how to best probe the network.

\section{The unconstrained problem}

In the introduction we presented Problem $1$: how to choose $x_{t}$
in order to minimize the error of the estimators. After observing
another example we shall examine the model more closely:
\begin{example}
Define the following problem:
\[
X=\left\{ \left(1,1,3\right),\left(1,1,0\right),\left(-2,-2,5\right)\right\} ,\; w\sim N\left(\left(\mu_{1},\mu_{2},\mu_{3}\right),diag\left(1,1,2\right)\right)
\]

In this example, it is possible to sample at each time step one of
the following linear combinations: $w\left(1\right)+w\left(2\right)+3w\left(3\right),\; w\left(1\right)+w\left(2\right),\;-2w\left(2\right)-2w\left(2\right)+5w\left(3\right)$.
\end{example}
Apparently, not all entries of $\mu$ can always be estimated with
a finite error: since all possible linear combinations include some
multiplication of the expression $w\left(1\right)+w\left(2\right)$,
adding a constant to $\mu_{1}$ and subtracting it from $\mu_{2}$
will not change the probability of the measurements and therefore
is undetectable, which implies a possibly infinite estimation error
for each. While the most logical way of handling this situation here
might be defining a new variable $\tilde{w}_{1,2}=w\left(1\right)+w\left(2\right)$,
in more complex cases it is not entirely clear how new variables should
be defined. Therefore, throughout the rest of the paper, unless specified
otherwise, we shall assume this situation does not occur, however
it still must be taken into account. 

Since all the variables in this model form a normal multivariate vector,
finding the MVUE (Minimum Variance Unbiased Estimator) which is also
the MLE (Maximum Likelihood Estimator) \cite{chen2011mathematical}
$\hat{\mu}$ is straightforward. Let $\Gamma$ be the $T\times N$
matrix whose $t^{th}$ row is $x_{t}$, $\sigma_{F,t}^{2}=x_{t}^{\top}diag\left(\sigma_{i}^{2}\right)_{i=1}^{N}x_{t}$
is the variance of the $t^{th}$ functional and $\Sigma_{\Gamma}$
is the diagonal matrix $\Sigma_{\Gamma}=diag\left(\sigma_{F,t}^{2}\right)_{t=1}^{T}$.
The following proposition is taken from Pukelsheim \cite{pukelsheim2006optimal}:
\begin{prop}
The inverse Fisher information matrix which is also the MSE matrix
for the MVUE estimator is given by:
\[
MSE\left(\mu\right)\triangleq E\left(\mu-\hat{\mu}\right)\left(\mu-\hat{\mu}\right)^{\top}=M^{-1}\triangleq\left(\sum_{t=1}^{T}\frac{1}{\sigma_{F,t}^{2}}x_{t}x_{t}^{\top}\right)^{-1}=\left(\Gamma^{\top}\Sigma_{\Gamma}^{-1}\Gamma\right)^{-1}.
\]

\end{prop}
As the MSE is a matrix, we would like to choose some scalar scoring
function to minimize. There are several suitable options (see Pukelsheim
\cite{pukelsheim2006optimal}), but we believe the simplest analytically
and most appropriate option is finding A-optimality, i.e. minimizing
the trace of the estimator's covariance matrix $M^{-1}$. 

Instead of solving the discrete time setup, we shall identify the
optimal decision policy as some stationary distribution on $X$ .
To ease the notation we will assume from now on that the random variables
have unit variance (i.e. $w_{t}\left(i\right)\sim N\left(\mu_{i},1\right)$).
In addition, we shall restrict $X$ to be a finite set for tractability
reasons. Denote by $\Delta_{N}$ the simplex set in $N$ variables,
i.e. $\Delta_{N}=\left\{ v\in\mathbb{R}_{+}^{N}|\mathbf{1^{\top}}v=1\right\} $.
The problem can be formulated as follows:
\begin{problem}
Find the optimal distribution $P$ on linear combinations from $X$
that achieves:
\[
P=\arg\min_{P\in\Delta_{N}}tr\left(\left[\sum_{x\in X}\frac{p\left(x\right)}{x^{\top}x}xx^{\top}\right]^{-1}\right).
\]
\end{problem}
\begin{rem}
Pukelsheim \cite{pukelsheim2006optimal} and Cesa-Bianchi and Lugosi
\cite{bianchi2009combinatorial} have formulated a different problem
for which the factor $\frac{1}{x^{\top}x}$ does not appear in each
summand. This is due to the slightly different setup: Pukelsheim had
defined that samples have the same variance for each functional, while
in our setup it is constant per coordinate, but functional dependent. 
\end{rem}
To simplify notation from now on, we abuse our previous notation by
redefining $\Gamma$ as the $\left|X\right|\times N$ matrix whose
rows are the distinct $x\in X$. Moreover, we define the matrices
$L=diag\left(x^{\top}x\right)_{x\in X}$ and $P=diag\left(p\left(x\right)\right)_{x\in X}$
so $M=\Gamma^{\top}L^{-1}P\Gamma$, and we want to minimize $tr\left(M^{-1}\right)$.
Evidently Problem $5$ can be fitted in a standard form as seen in
\cite{boyd2004convex}.
\begin{cor}
$Tr\left(M^{-1}\right)$ is a convex function of $p\left(x\right)$
and Problem $5$ can be solved using SDP (Semi-Definite Programming).\end{cor}
\begin{rem}
There is a minor variation on Problem $5$ that can be handled similarly:
consider the same objective function, only that now each functional
$x$ is associated with a cost $c\left(x\right)$ and there is some
restricted budget $C$. Adding the linear constraint $\sum_{x\in X}c\left(x\right)p\left(x\right)\leq C$
to the formulation does not affect its solvability using SDP. 
\end{rem}
According to Corollary $7$, solving Problem $5$ can be done in polynomial
time as a function of $\left|X\right|$. However, when $X$ is very
large it is unfeasible. Nevertheless, in practice large sampling spaces
tend to contain some inner structure and this is our motivation. We
view graph walks as structured sets for which a sub-optimal yet efficient
solution is employed.

\section{Binary functionals}

Binary functionals, i.e., linear combinations with coefficients only
in $\left\{ 0,1\right\} $, are an important and interesting subset
of possible functionals, since they are sufficient to describe sampling
in special models such as graphs. The meaning of using binary functionals
is that you choose which of the $N$ elements are part of your sample.
We start with the case where a subset of size $K$ of the variables
is chosen.

\subsection{K-choose-N}

The most natural set of binary functionals is the set of all functionals
with exactly $K$ ones. For example, for $K=1$ we get $X=\left\{ e_{i}\right\} _{i=1}^{N}$
and for $K=N$ we get $X=\left\{ \mathbf{1}\right\} $ ($\mathbf{1}$
denotes the vector of $N$ ones). It turns out the optimal solution
for these sets can be found analytically, as well as the solution
for unions of K-choose-N sets for different values of K. 
\begin{defn}
Denote the K-choose-N set $B_{K}$ by $B_{K}\triangleq\left\{ x\in\left\{ 0,1\right\} ^{N}|\; x^{\top}\mathbf{1}=K\right\} $.
\end{defn}
The following theorem determines the optimal solution for Problem
$5$ when $X=B_{K}$: 
\begin{thm}
Let $X=B_{K}$. The optimal solution for Problem $5$ is choosing
uniformly functionals over $X$, and the optimal MSE is given by:
$tr\left(M^{-1}\right)=\frac{N}{K}+\frac{\left(N-1\right)^{2}N}{N-K}.$\end{thm}
\begin{proof}
First we find the trace of the uniform decision for which $M=\frac{1}{K\left(\begin{array}{c}
N\\
K
\end{array}\right)}\sum_{x\in B_{K}}xx^{\top}.$ By applying a counting argument we obtain $M_{i,i}=\frac{1}{N}$,
$M_{i,j}=\frac{K-1}{N\left(N-1\right)}$ and $tr\left(M^{-1}\right)=\frac{N}{K}+\frac{\left(N-1\right)^{2}N}{N-K}.$

Now assume $M$ is A-optimal. Due to the symmetry of $B_{K}$ for
each variable, for any permutation matrix $\Phi$ we have $\Phi^{T}M\Phi\in conv\left(B_{K}\right)$.
From the convexity of Problem $5$ we can conclude that:
\[
tr\left(\left(\frac{\sum_{\Phi\in S_{N}}\Phi^{T}M\Phi}{N!}\right)^{-1}\right)\leq\frac{\sum_{\Phi\in S_{N}}tr\left(\left(\Phi^{T}M\Phi\right)^{-1}\right)}{N!}=tr\left(M^{-1}\right).
\]

Due to symmetry the matrix $M'=\frac{1}{N!}\sum_{\Phi\in S_{N}}\Phi^{T}M\Phi$
has constant diagonal entries and constant off-diagonal entries denoted
$c_{diag},\; c_{off}$ respectively. Since $tr\left(M'\right)=1$
we know that $c_{diag}=\frac{1}{N}$. As $M'\in conv\left(B_{K}\right)$
we have $\mathbf{1}^{T}M'\mathbf{1}=Nc_{diag}+N\left(N-1\right)c_{off}=K$,
so $c_{off}=\frac{K-1}{N\left(N-1\right)}$ and $M'$ is the same
matrix obtained by uniform choice.\end{proof}
\begin{example}
for $K=N-1$ we get $tr\left(M^{-1}\right)=\frac{N}{\left(N-1\right)}+\left(N-1\right)^{2}N$,
an $N^{3}$ asymptotic behavior.
\end{example}
Since the best value of $K$ is $1$, and for $K=N-1$ we got an error
that scales like $N^{3}$. We generalize this notion that smaller
$K$ yields better results:
\begin{cor}
If $K_{1}<K_{2}$, then the optimal solution of Problem $5$ for $X_{1}=B_{K_{1}}$
is smaller and therefore better than the optimal solution of Problem
$5$ for $X_{2}=B_{K_{2}}$.\end{cor}
\begin{proof}
Obtained from analysing $tr\left(M^{-1}\right)=\frac{N}{K}+\frac{\left(N-1\right)^{2}N}{N-K}$
as a function of $K$.
\end{proof}
So far we have shown the optimal solution for N-choose-K sets, and
in Corollary $12$ we also show that sets with smaller $K$ can be
used better. This result can be strengthened by the subsequent theorem
that suggests that if $X$ contains several N-choose-K subsets, only
the smallest subset is used for the optimal solution. In addition,
it gives rise to a general lower bound on binary functionals:
\begin{thm}
Assume $X=\bigcup_{i=K}^{N}B_{i}$, i.e., $X$ is the set of all linear
combinations with at least $K$ ones. The optimal solution of Problem
$5$ is given by a uniform choice over the functionals in $B_{K}$.\end{thm}
\begin{proof}
Like we showed in the proof of Theorem $10$, there is an optimal
matrix $M$ with constant off diagonal entries and due to its unit
trace and symmetry its diagonal entries are $\frac{1}{N}$. The smallest
off diagonal constant yields the minimal $tr\left(M^{-1}\right)$
so choosing the smallest $K$ is optimal. 
\end{proof}
\textbf{Conclusion:} For Problem $5$, if $X\subseteq\bigcup_{i=K}^{N}B_{i}$,
meaning all functionals in $X$ have at least $K$ ones, then $tr\left(M^{-1}\right)\geq\frac{N}{K}+\frac{\left(N-1\right)^{2}N}{N-K}$.

\section{Graph paths with randomness in the nodes}

Given a source-drain DAG (Directed Acyclic Graph) with $N$ inner
nodes, and assume $V=\left\{ v_{s},v_{d}\right\} \bigcup\left\{ v_{i}\right\} _{i=1}^{N}$
where the order over the nodes is defined by a topological order.
Each inner node is associated with a normally distributed random variable
$w\left(i\right)\sim N\left(\mu_{i},1\right)$, with an unknown $\mu_{i}$.
We would like to estimate the $\mu_{i}$'s with minimal MSE. This
scheme can model for example networks with delays generated from the
networking equipment in each node, but with constant or very low variance
link delays, e.g., optical networks. Denote by $x$ both the actual
path in the graph, and the corresponding characteristic vector, i.e.,
$x\left(i\right)=1$ iff $v_{i}\in x$. 
\begin{example}
Consider the following source-drain DAG: 
\end{example}
The possible paths on the graph allow us to sample the following linear
combinations: $w_{1}+w_{2}+w_{3},w_{1}+w_{2},w_{1}+w_{3},w_{2},w_{2}+w_{3}$,
so the corresponding $\Gamma$ matrix is: $\Gamma^{T}=\left[\begin{array}{ccccc}
1 & 1 & 1 & 0 & 0\\
1 & 1 & 0 & 1 & 1\\
1 & 0 & 1 & 0 & 1
\end{array}\right]$.

The following example exhibits many problems in the model we must
take into account:
\begin{example}
Consider a grid graph of the following form: 
\end{example}
This kind of graph is a good example for what can happen when ignoring
the complexity of Problem $5$ and instead uniformly choosing functionals
from the given set: if all paths on the grid are chosen the same number
of times, then the nodes in the middle will be sampled much more often
than these far on the sides, since many more paths go through them.
In their paper, Cesa-Bianci and Lugosi \cite{bianchi2009combinatorial}
have also addressed this counter example to the good results of uniform
distribution over the trajectories of many other graphs and models.
They suggested in their paper to find a better solution using semi-definite
programming which is impractical for non-trivial grids.

The number of paths in the example is large, which makes finding the
solution unfeasible for many nodes. Another concern we have neglected
so far that emerges in this example is its identifiability: in this
grid graph, adding a constant to the mean of all nodes at a certain
layer (meaning all nodes at the same distance from the source) and
subtracting the same constant from the mean of all nodes at another
layer will not change the distribution of the samples, so the set
of possible paths is unidentifiable. Apparently this is a key problem
in any layers graph, and although there are some reasonable suggestions
for dealing with this issue we shall neglect it in this paper as it
draws us further from the main scope.

Since the number of paths can be exponential in the number of nodes,
there might be too many functionals to optimally find the MSE using
SDP solvers. To cope with this setback, we offer a relaxed solution
that can be computed efficiently using dynamic programming. Simulations
show that our approach works quite well.

\subsection{The product distribution}

We propose a relaxed solution using dynamic programming by introducing
the product problem: observe only the distribution on paths generated
as the product of the leaving distribution from each node. More specifically,
denote $\alpha_{i,j}$ as the probability to leave the vertex $v_{i}$
using the edge $e_{i,j}$, so we get the following equations: $\sum_{j=i+1}^{N}\alpha_{i,j}+\alpha_{i,d}=1,\; p\left(x\right)=\prod_{i,j:e_{i,j}\in x}\alpha_{i,j}$.
Now we can define the relaxed problem which we later show is easier
to solve:
\begin{problem}
Find the optimal exit distributions $\alpha$ that solves the following
problem:
\[
\begin{cases}
\min_{\alpha} & tr\left(\left[\sum_{x\in X}\frac{p\left(x\right)}{x^{\top}x}xx^{\top}\right]^{-1}\right)\\
{\rm subject\; to} & p\left(x\right)=\prod_{i,j:e_{i,j}\in x}\alpha_{i,j},\;\alpha_{i,j}\geq0,\;\sum_{j=i+1}^{N}\alpha_{i,j}+\alpha_{i,d}=1
\end{cases}.
\]
\end{problem}
\begin{example}
Recall the graph from Example 

The value on each edge represents the exit distribution from its source
node so the distribution over paths is given by:
\[
p\left(v_{s}\rightarrow v_{1}\rightarrow v_{2}\rightarrow v_{d}\right)=0.7\cdot0.2\cdot0.5=0.07,\; p\left(v_{s}\rightarrow v_{2}\rightarrow v_{3}\rightarrow v_{d}\right)=0.3\cdot0.5\cdot1=0.15,...
\]

\end{example}
Notice that the set of product distributions is a subset of all possible
distributions over the paths, so the optimal product distribution
may produce a much worse MSE than the optimal unconstrained distribution.
However, since for optimization on the exit distributions $\alpha$
we got no more than $N^{2}$ variables, if $M$ can be expressed efficiently
using $\alpha$ then the computation effort will be drastically reduced.
In order for Problem $16$ to have an efficient solution we need to
be able to calculate the matrix $M$ without directly calculating
$p\left(x\right)$ for each $x$. In Theorem $18$ we show how it
can be done:
\begin{thm}
The matrix $M\left(\alpha\right)$ can be computed in polynomial time
using dynamic programming.\end{thm}
\begin{proof}
First we show how one can compute how many times length $l$ paths
contain each node. We calculate for each node sequentially (according
to a topological sort) how many times paths of length $l$ from the
source finish in this node using the following equation: $\Phi_{i}\left(l\right)=\sum_{j=1}^{i-1}\alpha_{j,i}\Phi_{j}\left(l-1\right)+\alpha_{s,i}1\left\{ l=1\right\} $.
Likewise, we can calculate for each node sequentially how many length
$l$ paths started in it and finished at the drain by employing the
following equation: $\Theta_{i}\left(l\right)=\sum_{j=i+1}^{N}\alpha_{i,j}\Theta_{j}\left(l-1\right)+\alpha_{i,d}1\left\{ l=1\right\} $.
Now the number of length $l$ paths that passed through node $i$
is given by: $J_{i,i}\left(l\right)=\sum_{j=1}^{l-1}\Phi_{i}\left(j\right)\Theta_{i}\left(l-j\right)$.
In a similar fashion we can compute the number of length $l$ trajectories
that passed through both node $i$ and node $j$: $J_{i,j}\left(l\right)$.
Finally, realizing that the matrix $J\left(l\right)$ satisfies $J\left(l\right)=\sum_{x:x^{T}\mathbf{1}=l}p\left(x\right)xx^{\top}$,
we can compute $M$ as the weighted sum of the matrices $\left\{ J\left(l\right)\right\} _{l=1}^{N}$:
$M=\sum_{l=1}^{N}\frac{1}{l}J(l)$.
\end{proof}
Obviously, the matrix $M$ linearly depends on each distinct set of
leaving probabilities $\left\{ \alpha_{i,j}\right\} _{j=i+1}^{d}$
or entering probabilities $\left\{ \alpha_{i,j}\right\} _{i=s}^{j-1}$,
while not changing the other values of $\alpha$. Therefore for each
such set of variables the problem of minimizing $tr\left(M^{-1}\right)$
is an SDP in that specific set of variables. However, over the entire
set of variables $\left\{ \alpha_{i,j}\right\} $ the function is
not an SDP as it is not convex. Therefore, even if we compute for
each node iteratively the optimal exit distribution assuming all the
other distributions are constant, we cannot be assured the solution
found is globally optimal. Algorithm $1$ describes this general scheme.
Notice that the objective function decreases at each iteration so
convergence is guaranteed. The order in which the nodes are chosen
can play a role in the convergence rate; we leave that aspect for
future research. Empirically the optimal solution is unique and closely
approximates the solution of Problem $5$, as seen in Figure $5.2\left(A\right)$. 

\begin{wrapfigure}[20]{o}{0.27\columnwidth}%
\caption{The optimal product distribution for $5\times5$ grid. Notice the
weights give higher probability to go through the corners than the
uniform distribution.}

\raggedleft{}\end{wrapfigure}%

In Figure $5.2\left(B\right)$ we compare our relaxed solution against
the uniform distribution for the grid graph. Note that for square
grids with $a$ nodes on each side there are $\left(\begin{array}{c}
2a-2\\
a-1
\end{array}\right)$ paths in the graph, so the optimal solution is impossible to compute
for large $a$. To deal with the identifiability problem, instead
of $tr\left(M^{-1}\right)$ we used the objective function $\sum\limits _{\lambda_{i}\neq0}\frac{1}{\lambda_{i}}$
to minimize. The results for the product distribution are much better
than these obtained by the uniform distribution. 

\begin{flushleft}
\begin{algorithm}
\begin{raggedright}
\caption{Optimize\_Products}

\par\end{raggedright}

\begin{raggedright}
1. Start with random exit distributions for each node $\left\{ \alpha_{i,j}\right\} $.
\par\end{raggedright}

\begin{raggedright}
2. Choose node $v$.
\par\end{raggedright}

\begin{raggedright}
3. Find the optimal exit distribution $\left\{ \alpha_{v,j}\right\} _{j=v+1}^{d}$
from node $v$ to all other nodes assuming all other exit distribution
$\left\{ \alpha_{i,j}\right\} _{i\neq v}$ are constant.
\par\end{raggedright}

\raggedright{}4. Go to step $2$ with a different node.
\end{algorithm}

\par\end{flushleft}

\begin{figure}
\caption{Matlab Simulations. For the implementation we used the CVX package
\cite{cvx,gb08}. }

\raggedright{}\subfloat[Uniform distribution and optimal product solution divided by the optimal
solution. We simulated graphs in the following manner: first create
a path graph from the source to the drain through each node, now add
each forward edge with probability $0.5$. If the generated DAG is
unidentifiable we discard it and sample another random graph. ]{\begin{raggedright}

\par\end{raggedright}

\raggedright{}}\subfloat[Grid simulation over square grids: uniform sampling vs. product solution.
Smaller values are better as they indicate smaller error. ]{\begin{raggedleft}

\par\end{raggedleft}

\raggedleft{}}
\end{figure}

\section{Graph paths with randomness on the edges}

Although we initially acknowledged DAG graphs in which the randomness
is associated with the vertices, it is common in application to associate
them with the edges, for example as delays in a network. To fit our
model to such applications we shall now assume a graph with multiple
access points from which the user can probe the network to another
access point. Even though such graphs will not necessarily be DAGs,
we will not allow cycles as they are not usually allowed in regular
networks and adversely affect the estimation since they just add more
noise. 

Let $G=\left(V,E\right)$ be a simple graph where each edge is associated
with a normal random variable with an unknown mean and unit variance.
In addition let $S\subseteq V$ be a set of access points in the graph.
Each time step, it is possible to choose a path in the graph starting
with one access point and ending in another. Finally, the sum of the
random variables over the edges in the path is presented, from which
one can estimate the expectation of the random variable associated
with each edge. Although a directed graph is more appropriate to describe
reality, in the next example we shall assume that the graph is undirected
for simplicity, which is equivalent to the claim that the delay in
each direction has the same distribution. 
\begin{example}
Consider a star graph with $v_{0}$ as its center vertex and assume
that all edges from and to the center exist. If $S=V\backslash\left\{ v_{0}\right\} $,
we get an identical case as N-choose-K, where $K=2$. As we saw, the
optimal solution here is uniform over all $2$ access points. Notice
that by giving uniform distribution from the center vertex to any
of the edges except the one of the root access point, the optimal
solution in this setting is obtained.
\end{example}
In order to cope with the exponential number of paths, we can define
here as well a product rule: for each node $v_{i}\in V$, and its
set of exit edges $\left\{ e_{i\rightarrow j}\right\} _{j:v_{j}\in{\rm Neighbors}\left(v_{i}\right)}$
define an exit distribution $\alpha_{i\rightarrow j}$ as the probability
to take the edge $e_{i\rightarrow j}$ from node $v_{i}$. 
\begin{thm}
The matrix $M\left(\alpha\right)$ can be computed in polynomial time
complexity using dynamic programming.\end{thm}
\begin{proof}
In a similar fashion to the proof of Theorem $22$ we can calculate
for each access point and for each edge $e_{i\rightarrow j}$ in the
graph by dynamic programming the number of $k$-length paths that
began at that access point and ended at vertex $i$. Similarly we
can calculate the number of $k$-length paths that began at vertex
$j$ and ended in that access point. Convolving the results provides
us with the number of $k$-length paths that passed through the edge
$e_{i\rightarrow j}$, and from that $M_{i\rightarrow j}$ is easily
obtained.
\end{proof}
Theorem $20$ allows us to use Algorithm $1$ for efficient computation
of optimal product solution for this case as well, so the product
solution can be used for efficient estimation of delays in networks.

\section{Online bandits}

Cesa-Bianchi and Lugosi \cite{bianchi2009combinatorial} have come
across a similar problem in the adversarial online bandit problem
with a restricted linear sampling set. They showed a performance bound
that depends on the lowest eigenvalue of the matrix $F=\sum_{x\in X}p\left(x\right)xx^{\top}$.
Maximizing the smallest eigenvalue is called in the literature E-criterion
and it can be formulated in SDP form. It is easy to see that $\left(\min_{x\in X}\left\Vert x\right\Vert ^{2}\right)M\preceq F\preceq\left(\max_{x\in X}\left\Vert x\right\Vert ^{2}\right)M$
(using the Lowener order for symmetric matrices), and that if all
vectors in $X$ have norm $b$ (like in the K-choose-N example or
the grid) then $F=bM$. This means that there is a close connection
between the two problems, especially for the binary case for which
$\frac{\max_{x\in X}\left\Vert x\right\Vert ^{2}}{\min_{x\in X}\left\Vert x\right\Vert ^{2}}\leq N$.
In that case, as stated by Theorems $18$ and $21$, relaxed solutions
can be found efficiently for Problem $5$ on the graph setups by considering
product distributions. For minimizing $F$ there is a similar result
as we show in the next Theorem for nodes-associated randomness (a
similar result can be shown for the case of edges):
\begin{thm}
The matrix $F\left(\alpha\right)$ can be computed in polynomial time
complexity using dynamic programming.\end{thm}
\begin{proof}
Denote by $n\left(j\right)$ the appearance frequency of the $j$'th
node and by $n\left(i,j\right)$ the joined appearance frequency of
nodes $i$ and $j$. Observe that $n\left(j\right)=\sum_{k=s}^{j-1}n\left(k\right)\alpha_{k,j},\; n\left(i,j\right)=\sum_{k=i}^{j-1}n\left(i,k\right)\alpha_{k,j}$.
So computing $F_{i,j}=n\left(i,j\right)$ is simply applying these
equations in the order they are written for incrementing values of
$j$. 
\end{proof}
According to Theorem $21$, Algorithm $1$ can be used for minimizing
$F$ on all product distributions efficiently as well. Therefore we
can use this algorithm to find and simulate sub-optimal exploration
distribution on the sampling space in the suggested bandit setup.

\section{Conclusions}

In this paper we considered a fundamental problem that is common in
many setups. Although a straightforward solution for the optimal sampling
problem exists, it might be unfeasible to compute. Therefore, for
graph paths we proposed an efficient relaxed solution that exploits
the graphical structure using dynamic programming. The suggested solution
was tested empirically and our simulations showed good behavior. In
addition we linked a recently suggested bandit setup with the field
of optimal experiments design, and employed our solution on the grid
example for which uniform sampling is inadequate.

Our paper opens up some interesting research directions. Among these
directions are: the case of an infinite set $X$, bounding the difference
between the relaxed product solution and the optimal one, finding
graph properties based bounds, and analyzing the behavior of random
graphs or sets in this context. 

\bibliographystyle{IEEEtran}
\bibliography{Refs/Problem1}

\end{document}